\documentclass{article}
\newtheorem{keywords}{Keywords}
\newtheorem{definition}{Definition}
\newtheorem{theorem}{Theorem}
\newtheorem{proof}{Proof}
\newtheorem{remark}{Remark}
\begin{document}
\title{On rough mereology and VC-dimension in treatment of decision prediction for open world decision systems}
\author{Lech T. Polkowski\\ Department of Mathematics and Informatics\\ University of Warmia and Mazury in Olsztyn Poland\\ e-mail: polkow@pja.edu.pl}
\date{}
\maketitle

\hspace{2cm} {\em In memory of Professor Helena Rasiowa (1917-1994)}

\begin{abstract}\ Given a raw knowledge in the form of a data table/a decision system, one is facing two possible venues. One, to treat the system as closed, i.e., its
universe does not admit new objects, or, to the contrary, its universe is open on admittance of new objects. In particular, one may obtain new objects whose sets of  values
of features are new to the system. In this case the problem is to assign a decision value to any such new object. This problem is somehow resolved in the rough set theory,
e.g., on the basis of similarity of the value set of a new object to value sets of objects already assigned a decision value.\ It is crucial for online learning when each new
object must have a predicted decision value.\ There is a vast literature on various methods for decision prediction for new yet unseen object, we  mention a few works below.

The approach we propose is founded in the theory of rough mereology and it requires a theory of sets/concepts, and, we root our theory in classical set theory of
Syllogistic within which we recall the theory of parts known as Mereology. Then, we recall our theory of Rough Mereology along with the theory of weight assignment to the
Tarski algebra of Mereology.\ This allows us to introduce the notion of a part to a degree.

Our choice of Syllogistic as the theory of concepts is motivated not only by simplicity of this theory and its completeness but also by our remembrance of Professor Helena
Rasiowa who studied in years 1939-1943 at Clandestine Warsaw University with Jan \L ukasiewicz, the great authority on Aristotle's Syllogistic.

Once we have defined basics of Mereology and rough Mereology, we recall our theory of weight assignment to elements of the Boolean algebra defined within Mereology and this
allows us to define the relation of parts to the degree and we apply this notion in a procedure to select a decision for new yet unseen objects.\ In selecting a plausible
candidate which would  pass its decision value to the new object, we employ the notion of Vapnik - Chervonenkis dimension in order to select at the first stage the
candidate with the largest VC-dimension of the family of its $\varepsilon$-components for some choice of $\varepsilon$.

Our text is, accordingly, divided into sections which are, respectively:\

1.\ The section on basic notions of the rough set theory.

2.\ The section on Syllogistic.

3.\ The section on Mereology.

4.\ The section on weight assignment.

5.\ The section on Rough Mereology.

6.\ The section on our idea for the treatment of new objects in the Open World Approach (OWA).

7.\ On the problem of mistakes.
\end{abstract}

\begin{keywords} Syllogistic, Mereology, Rough Mereology, Rough Set Theory, decision systems, Open World assumption, Vapnik-Chervonenkis dimension.\end{keywords}

\section{The Rough Set Theory: Basic notions}\label{RST} We begin with a decision system $\mathcal{D}$ = $(O, F,d,V, V_d)$ in which $O$ is a set of objects, $F$ is a set of
features, $d\notin F$ is the decision feature,  $V$ is the set of feature values, and, $V_d$ is the set of decision values. The set $F$ consists of {\em subjective features}
of our free choice. We assume the set $V_d$ to be a convex set in a finite dimensional vector space, in particular it may be the interval in the real line.

We adopt the latter as our standard model. Rough Set Theory due to Pawlak \cite{Pawlak} aims at description of concepts - subsets of the set $O$ and a fortiori at assigning
to decision values their descriptions in terms of values of subjective features. Expressions which do this work are {\em decision rules}, their sets forming decision
algorithms, particular cases of Datalog cf. Abiteboul et al. \cite{Abiteboul}. The premiss of a decision rule is a conjunct of descriptors of the form $\bigwedge_i (f_i,
v_i)$ where $f_i$ is a subjective feature and $v_i$ its value. The conclusion is a descriptor of the form $(d, v)$ where $v$ is a decision value, features $f_i$ run over a
subset $R$ of the set $F_1$ and the decision rule in question is of the form $r: \bigwedge_{f\in R} (f, v_f)\rightarrow (d,v)$.\ The arrow symbol $\rightarrow$ is used only
in the context of decision rules, otherwise and in general we use the horseshoe symbol $\supset$.\

Concepts definable as subsets of the set of objects $O$ are unions of indiscernibility classes of the form $IND_H(o)=\{o^{\prime}\in O: f(o)=f(o^{\prime})\}$ for $o\in O$ and
$f\in H$ for $H\subset F_1$; other concepts are non-definable (rough) and they may be defined only approximately by Jordan - type approximations by definable sets.

The interested reader may consult for rough set theory proceedings of the annual series of IJCRS conferences (LNCS, Springer-Verlag) and issues of Transactions on Rough Sets
\cite{TRS}; an account of foundations may be found in Pawlak \cite{Pawlak}, \cite{Pawlak 1}, Pawlak et al. \cite{PPS}, Polkowski \cite{Polkowski 0}, or, Komorowski et al.
\cite{Komor}. Among works dedicated to prognosis of decision values for new yet unseen objects we may mention Chen et al., \cite{Chen}, D\"{u}ntsch and Gediga \cite{Dun},
Udristoiu et al. \cite{Udr} and Wu et al. \cite{Wu}. Rough set theory-based analysis and approximation of decision is the subject in Pawlak-S\l owi\'{n}ski \cite{PS}, Yao
\cite{Yao}, Greco et al. \cite{GMS}, Jian et al. \cite{Jian}, Polkowski and Skowron \cite{PoS}, \cite{PoS1}, Pal et al.\cite{Pal}, Polkowski and Artiemjew \cite{PA}, among
many others.

\section{Syllogistic in a nutshell} We choose as our theory of collective terms the Aristotle Syllogistic presented originally in {\em Analytics} and translated and edited,
e.g.,  in Ross \cite{Ross} and we borrow our description of Syllogistic from \L ukasiewicz \cite{Lukasiewicz 1}. The reason for this choice is given in Abstract. We use the
expression 'term' as it is used in analysis of the Aristotle's system of Syllogistic in \cite{Lukasiewicz 1}. Syllogistic involves terms in premisses in which one term is
predicated of the other term like in the expression 'Europeans are predicated of all Poles' where 'Poles' is the subject and 'Europeans' is the predicate. The expression
'All' makes the premiss into a universal one. Other qualifiers to premisses are 'No', 'Some'. 'No' qualifies a premiss as universal and 'Some' qualifies a premiss as
particular. Let us note that all predicates are mass expressions like 'man', 'Pole', 'human being'. Some divergent opinions concern subjects,  \L ukasiewicz takes  the view
that subjects should also be collective terms,  some other authors like Smith \cite{Smith} allow singular terms as subjects as in the  syllogism: 'If all men are mortal and
Socrates is a man then Socrates is mortal'. We make use of this possibility later on.

 Analysis by \L ukasiewicz led to the formulation of expressions of Syllogistic as syllogisms rendered as implications  like 'If All Poles are Europeans and All Warsawians
 are Poles then All Warsawians are Europeans'. Abstract forms of syllogisms like '*\ If All a is b and All b is c then All a is c' are moods divided into four Figures. The
 form of a mood is not uniquely adopted in the literature. We follow the version by \L ukasiewicz with the subject before the predicate, some authors, e.g., Smith
 \cite{Smith}  place the predicate before the subject. In some places in those moods 'All' is replaced by 'Some' or 'No', or, 'Some ... is Not ...' . Aristotle used in
 validity proofs methods of reduction of  moods to some 'obvious' moods, see \cite{Lukasiewicz 1}. Medieval scholars gave moods names like 'Barbara' for the mood $^*$ above
 and that name system hid methods of reduction cf. Lagerlund \cite{Lagerlund}.  One may also use Eulerian circles (alias Venn diagrams). Aristotle  established validity of 24
 moods out of possible 256. Those 24 moods provide a complete description of relations among mass collections of things. We introduce Syllogistic in the axiomatic form due to
 \L ukasiewicz \cite{Lukasiewicz 1}.\

The formal notation for premisses is:\

$Aab$: All $a$ are $b$.

$Iab$: Some $a$ is $b$.

$Oab$: Some $a$ is Not $b$.

$Eab$: No $a$ is $b$.\

Constants $A,I$ come from the first two wovels in 'affirmo' and constants $E,O$ are the vowels in 'nego'. Premisses are made into moods of Syllogistic written down for short
by means of sentential sentence-making functors of conjunction $\wedge$ and implication $\supset$.

\begin{definition}\label{mood} [A logical form of a mood] For $X, Y, Z \in \{A, I, O, E\}$, a mood is of the form

\begin{equation}\label{moodform}
Xam\wedge Ymb\supset Zab
\end{equation}
\end{definition}

 The deductive scheme for selection of all valid moods was proposed in \L ukasiewicz \cite{Lukasiewicz 1}. We recall that scheme.\

{\bf Axiom schemes}

(T1)\ $Aaa$.

(T2)\ $Iaa$.

(T3)\ $Amb\wedge Aam\supset Aab$.

(T4)\ $Amb\wedge Ima\supset Iab$.\

{\bf Inference rules}\

(D)\ Detachment.
(S)\ Substitution of a sentential variable or a term for equiform occurrences of a variable.

(R)\ Replacement by equivalents: $(R1):\ Oab\equiv \neg Aab$, $(R2):\ Eab\equiv \neg Iab$.\

Derivations of moods require some laws of contradiction, sub-alternation or  conversion.

Exemplary rules for derivations are \cite{Lukasiewicz 1}:\

(L1)\ The law of contradiction: $Eab\supset \neg Iab$: {\em if No a is b, then it is not true that Some a is b}.

Proof:
1.\ In the theorem $p\supset p$ of sentential calculus, substitute $p/\neg Iab$ and apply the replacement rule (R2).\

(L2)\ The law of sub-alternation: $Aab\supset Iab$: {\em if All a is b, then Some a is b}.

Proof:
1.\ In the theorem $((p\wedge q)\supset r) \supset (q\supset (p\supset r))$ of the sentential calculus, substitute $p/Aab$, $q/Iaa$, $r/Iab$. The obtained formula is $\phi$.

2.\ In (T4), substitute $m/a$, let the obtained formula be $\psi$.

3.\ Detach $\psi$ from $\phi$, call the obtained formula $\chi$.

4.\ Detach  (T2) from $\chi$.\

(L3) The law of conversion: $Iab\supset Iba$: {\em if Some a is b, then Some b is a}.

Proof:
1.\ In the theorem $((p\wedge q)\supset r) \supset (p\supset (q\supset r))$ of the sentential calculus, substitute $p/Aaa$, $q/Iab$, $r/Iba$, call the obtained formula
$\phi$.

2.\ In (T4), substitute $m/a$, $a/b$, $b/a$, call the obtained formula $\psi$.

3.\ Detach $\psi$ from $\phi$, call the obtained formula $\chi$.

4.\ Detach (T1) from $\chi$.\

(L4)\ The law of conversion: $Aab\supset Iba$ follows from (L2) and (L3) by the hypothetical syllogism.\

We may now include after \cite{Lukasiewicz 1} an exemplary proof of the mood Barbari: $Amb\wedge Aam\supset Iab$.\

Proof of Barbari: 1.\ In the theorem $((p\wedge q)\supset r) \supset ((s\supset q)\supset ((p\wedge  s)\supset r)$ of the sentential calculus, substitute $p/Amb$, $q/Ima$,
$r/Iab$, $s/Aam$. Let the result be $\phi$.

2.\ Detach (T4) from $\phi$ and call the result $\psi$.

3.\ IN (L4), substitute $b/m$. The result is Barbari.\

In analogous veins, one obtains 22 valid moods which along with (T3) (Barbara) and (T4) (Datisi) make up for all 24 valid moods of Syllogistic.
This leaves 232 moods verified by Aristotle (loc.cit) as invalid.\

\L ukasiewicz, see \cite{Lukasiewicz 1}, produced an axiomatic system for rejection of moods. Its axiom schemes are as follows; the asterisk means: the formula is rejected.

\begin{enumerate}
\item{F1}\ $*\ Acb \wedge Aab \supset Iac$.
\item{F2}\ $*\ Ecb \wedge Eab \supset Iac$.
\end{enumerate}

Rejection rules are the following.\

(F3)\ Rejection by detachment:\ if the implication $p\supset q$ is accepted and $q$ is rejected, then $p$ is to be rejected.

(F4)\ Rejection by substitution:\  if $q$ is obtained by a valid substitution from $p$  and $q$ is rejected, then $p$ is to be rejected. \

Although it is traditional to list moods divided into Figures, yet we list all valid moods by the conclusions. We only list premisses  for the given conclusion. After the
name of a mood, we give the negation of the mood, for a further purpose.\

(M1)\ For the conclusion $Aab$.\

(M1.1)\ $Amb\wedge Aam$. [Barbara]\ $Amb\wedge Aam\wedge Oab$.\

(M2)\ For the conclusion $Iab$.\

(M2.1)\ $Amb, Ima$. [Datisi]\  $Amb\wedge Ima\wedge Eab$.\

(M2.2)\ $Amb, Aam$. [Barbari]  $Amb\wedge Aam\wedge Eab$.\

(M2.3)\ $Amb, Iam$. [Darii]\   $Amb \wedge Iam\wedge Eab$.\

(M2.4)\ $Amb, Ama$. [Darapti]\ $Amb \wedge Ama\wedge Eba$.\

(M2.5)\ $Imb, Ama$. [Disamis]\ $Imb \wedge Ama \wedge Eba$.\

(M2.6)\ $Abm, Ama$. [Bamalip]\ $Abm \wedge Ama \wedge Eba$.\

(M2.7)\ $Ibm, Ama$. [Dimatis]\ $Ibm \wedge Ama \wedge Eba$.\\

(M3)\ For the conclusion $Eab$.\\

(M3.1)\ $Emb, Aam$. [Celarent]\ $Emb \wedge Aam\wedge Iab$.\

(M3.2)\ $Ebm, Aam$. [Cesare]\  $Ebm \wedge Aam \wedge Iab$.\

(M3.3)\ $Abm, Eam$. [Camestres]\ $Abm \wedge Eam \wedge Iab$.\

(M3.4)\ $Abm, Ema$. [Calemes]\  $Abm \wedge  Ema \wedge Iab$.\\

(M4)\ For the conclusion $Oab$.\\

(M4.1)\ $Emb, Aam$. [Celaront]\ $Emb\wedge Aam \wedge Aab$.\

(M4.2)\ $Emb, Iam$. [Ferio]\ $Emb \wedge Iam \wedge Aab$.\

(M4.3)\ $Ebm, Aam$. [Cesaro]\ $Ebm \wedge Aam \wedge Aab$.\

(M4.4)\ $Abm, Eam$. [Camestrop]\ $Abm \wedge Eam \wedge Aab$.\

(M4.5)\ $Ebm, Iam$. [Festino]\ $Ebm \wedge Iam \wedge Aab$.\

(M4.6)\ $Abm, Oam$. [Baroco]\  $Abm \wedge Oam \wedge Aab$.\

(M4.7)\ $Emb, Ama$. [Felapton]\ $Emb \wedge Ama \wedge Aab$.\

(M4.8)\ $Omb, Ama$. [Bocardo]\  $Omb\wedge Ama \wedge Aab$.\

(M4.9)\ $Emb, Ima$. [Ferison]\  $Emb \wedge Ima \wedge Aab$.\

(M4.10)\ $Abm, Ema$. [Camelop]\ $Abm \wedge Ema \wedge Aab$.\

(M4.11)\ $Ebm, Ama$. [Fesapo]\  $Ebm \wedge Ama \wedge Aab$.\

(M4.12)\ $Ebm, Ima$. [Fresison]\ $Ebm \wedge Ima \wedge Aab$.\\

Completeness of Syllogistic was proved in S\l upecki \cite{Slupecki}, \cite{Slupecki 1}. The fact that terms may be interpreted as sets/concepts, prompted authors to transfer
the mechanisms of Syllogistic into set structures of modern logic, see Corcoran \cite{Corcoran 1}, Andrade-Becerra \cite{Andrade}, Moss \cite{Moss}. In this setting
completeness of Syllogistic  was proved, e.g.,  in Corcoran \cite{Corcoran} and in Moss \cite{Moss} who also considers some fragments of Syllogistic and gives Henkin-style
proofs of completeness of those fragments.\

Regarding Syllogistic as a formal system, we may  hint at yet another  proof of completeness. We include its idea.
\begin{definition}\label{Cons}[Consistency]\ A provable formula $\phi$ in an axiomatic system $\mathcal{A}$ is consistent if and only if its negation $\neg\phi$ is not
provable.
\end{definition}
\begin{definition}\label{model} [Models for moods]\ A model for a mood of Syllogistic is a set of 3 Eulerian circles each of them representing the respective premiss or
conclusion. \end{definition}
\begin{definition}\label{comp}[Completeness] An axiomatic system is complete if each valid formula is provable.\end{definition}
\begin{theorem}\ An axiomatic system is complete if and only if each consistent formula is true in some model.\end{theorem}
\begin{proof}\ Suppose a formula $\phi$ is  valid but not provable. Thus, the formula $\neg\neg \phi$ is not provable and it follows that the formula $\neg\phi$ is
consistent, hence, $\neg \phi$ has a model, contrary to the validity of $\phi$.\end{proof}
\begin{theorem}\ The axiomatic system of Syllogistic is complete and decidable.\end{theorem}

\begin{proof}\ Each of moods from (M1.1) to (M4.12) has its negation non-provable, hence, each of these moods is consistent. Each of these moods has an appropriate model of
three Eulerian circles which proves validity of all them. This concludes proof of completeness of the system of Syllogistic. As for decidability, it follows by the finite
model property of Syllogistic: each valid mood has a finite model. \end{proof}
We need now an account of Mereology.

\section{Mereology in a nutshell} We outline the theory of Mereology due to Le\'{s}niewski \cite{Lesniewski}. We set this rendition within Syllogistic theory of terms. It is
well known that the primitive notion of a part is definable in Syllogistic, see, e.g., Vlasits \cite{Vlasits}. Many authors understand part as containing also the notion of
the whole yet the original scheme in \cite{Lesniewski} defined part as the proper part distinct from the whole reserving for eventual part or the whole the term of an
ingredient. We call it the component.\ Consistency of Mereology was established in Lejewski \cite{Lej}.\

We express the term '{\em proper part}' by the formula $Pab$ read '{\em a is a proper part of b}':
\begin{equation}\label{PP}
Pab=_{df} Aab\wedge Oba
\end{equation}
It is accepted that Syllogistic can be extended by formulas with sequences of nested operators like $\alpha_1 \supset (\alpha_2\supset (... \supset (\alpha_n \supset
\beta)...)$ see, e.g., Lagerlund \cite{Lagerlund}; we will occasionally make use of this possibility. The equivalent definition of the proper part will run as follows.
\begin{definition}\label{PP1} [Constant of proper part]\
$Pab =_{df} Aam \supset (Amb \supset (Pab\supset Obm))$
\end{definition}
This definition conforms to the Le\'{s}niewski manner of defining: a definition is a formula, with  definiendum and definiens in it, which is assumed to be true and in that
way it reveals the sense of definiendum. Properties of $P$ follow.
\begin{theorem}\label{parts} The following are true.

1.\ It is not true that $Paa$, because $Oaa$ is false.

2.\ $Pab\wedge Pbc\supset Pac$, because by the valid syllogism $Aab\wedge Abc\supset Aac$, and def. \ref{PP} we obtain $Aac$; by def.\ref{PP}, we have $Oba$ and $Ocb$; was
$\neg Oca$ we would have $Aca$, hence $Acb$, hence, $\neg Obc$, a contradiction.

3.\ $Pab\supset \neg Pba$; by (1) and (2).
\end{theorem}

We next define the constant $=$ of identity.
\begin{equation}\label{eq}
=ab  =_{df} Aab\wedge Aba
\end{equation}
By properties of $A$, $=$ satisfies $=aa, =ab\supset =ba, =ab\wedge =bc\supset =ac$. \

The constant $C$ of a component is defined next.
\begin{equation}\label{C}
Cab =_{df} Pab \vee =ab
\end{equation}
\begin{theorem}\label{propC} The following are properties of components.\

1.\ $Caa$.

2.\ $Cab\wedge Cba\supset =ab$.

3.\ $Cab\wedge Cbc\supset Cac$.
\end{theorem}
\begin{proof} 1. follows by definition. 2. Of all possibilities: $(Pab\vee =ab)\wedge (Pba\vee =ba)$ only $=ab\wedge =ba$ is true. 3.\ Follows on lines similar to those of
2.\end{proof}
The mereological relation of overlap may be expressed by the constant $OV$.
\begin{definition}\label{OV} [The constant of overlapping] $OVab =_{df} Iab$. \end{definition}

\begin{theorem}\label{OVsym} $OVab\supset OVba\wedge OVba\supset OVab$.\end{theorem}
\begin{proof} Because of the equivalence $Iba$ if and only if $Iab$.\end{proof}

We have rendered basic notions of Mereology. The next set of notions are those of a topological tint.

\begin{definition}\label{EX} [The constant 'exterior' EX] The term $EXab$:  '{\em a is exterior to b}' is defined as $EXab =_{df} Eab$.\end{definition}

\begin{theorem}\label{EX1} $EXab$ if and only if $EXba$.\end{theorem}

\begin{proof} The thesis follows from the fact that $Eab$ is equivalent to $Eba$.\end{proof}

We are closing on the notion of an interior. First, we define the notion of relative exterior $Re_bam$ of $a$ to $m$ relative to $b$.\

\begin{definition}\label{REX} [The constant 'relative exterior' Re] $Re_bam =_{df} Pab\wedge Pmb \wedge EXam$.\end{definition}

\begin{theorem}\ $Re_bam$ if and only if $Re_bma$.\end{theorem}

\begin{proof} It follows from the equivalence of $EXam$ and $EXma$.\end{proof}

\begin{definition}\label{A} [Axiom A].\ $((Cma\supset Inb\wedge Cnb)\supset OVmn)\supset Cab$. In reading: {\em [if ($m$ is any component of $a$, then Some $n$ is $b$ and
this $n$ is a component of $b$), then $m$ and $n$ overlap], then $a$ is a component of $b$}.\end{definition}
Mereology of Le\'{s}niewski contains a discussion of the notion of a class as a collective term. The idea is that classes are defined for collections of terms making them
into singular terms. It is here that we allow singular subjects in formulae of Syllogistic.
\begin{definition}\label{class} [The constant 'Class'] For a collection $B$ of terms, the class of $B$, $ClaB$: read '{\em a is the class of B}' is defined by the two
following requirements:\

1.\ $b\ is\ B\supset Cba$.

2.\ $Cba\supset (IcB\wedge OVbc)$.
\end{definition}
As shown in Tarski \cite{Tar}, within any mereological universe there exists the structure of a complete Boolean algebra without the null element. Its operations are defined
as follows.
\begin{definition}\label{algebra} [The Boolean algebra in mereological universe]\ Its operations are defined as follows.\

1.\ $a+b=Cl\{t: Cta\vee Ctb\}$.

2.\ $a\cdot b=Cl\{t: Cta\wedge Ctb\}$.

3. $Cl\{all\ terms\}=V$.

4. $-a=Cl\{t: EXta\}$.\end{definition}

It is known that the Tarski algebra is complete due to class existence.\ $V$ is the universe of all terms.\

The implication $\hookrightarrow$ is defined as
\begin{definition}\label{hook} [The mereological implication] $x\hookrightarrow y = -x+y$. $x\hookrightarrow y$ is valid if and only if $-x+y=V$.\end{definition}

The implication $\hookrightarrow$ satisfies the standard properties valid in Boolean algebras:\

1.\ $(x\hookrightarrow y)\hookrightarrow ((y\hookrightarrow z)\hookrightarrow (x\hookrightarrow z))$.\

2.\ $x\cdot y)\hookrightarrow x$.\

3.\ $(x\cdot y)\hookrightarrow (y\cdot x)$.\

4.\ $(x\cdot (x\hookrightarrow y)) \hookrightarrow (y\cdot (y\hookrightarrow x))$.\

5.\ $(x\hookrightarrow (y\hookrightarrow z))\hookrightarrow ((x\cdot y)\hookrightarrow z)$.\

6.\ $((x\cdot y)\hookrightarrow z)\hookrightarrow (x\hookrightarrow (y\hookrightarrow z))$.\

7.\ $(x\hookrightarrow (y\hookrightarrow z))\hookrightarrow (((y\hookrightarrow x)\hookrightarrow z)\hookrightarrow z)$.\

One checks that the value of each formula among  1-7 is $V$, i.e., each of them is valid.\

\section{The weighted Tarski algebra}\ We endow the Tarski algebra with a weight function $m$ (called a mass assignment in Polkowski \cite{Polkowski 1}) with values in the
interval $(0,1]$. The function $m$ should conform to the following requirements.

\begin{definition}\label{weight} [The weight function] Concerning the function $m$:\

1.\ $((m(x)=1)\supset (x=V))\wedge ((x=V)\supset (m(x)=1))$.

2.\ $(x\hookrightarrow y)\supset ((m(y)=m(x)+m((-x)\cdot y))$.

3.\ $m(x)>0$ for each thing $x\in U$.

We extend $U$ by adding the empty term $e$ and declaring:

4.\ $m(e)=0$.
\end{definition}

From axiom schemes (1)-(4), we infer the following properties of the function $m$ cf. \cite{Polkowski 1}.

\begin{theorem} The function $m$ satisfies the following properties:\

(m1)\ $(Cxy\supset x\hookrightarrow y)\wedge (x\hookrightarrow y\supset Cxy)$.

(m2)\ $(Cxy\supset [(x\cdot y=x)\wedge (x\cdot y=x\supset Cxy)]$.

(m3)\ $(x=y)\supset ((m(x)=m(y))$.

(m4)\ $m(x+y)= (m(x)+m((-x)\cdot y))$.

(m5)\ $(x\cdot y=e)\supset (m(x+y)=m(x)+m(y))$.

(m6)\ $m(x)+m(-x)=1$.

(m7)\ $m(y)=m(y\cdot x)+m(y\cdot(-x))$.

(m8)\ $m(x+y)=m(x)+m(y)-m(x\cdot y)$.

(m9)\ $Cxy\supset (m(x)\leq m(y))$.

(m10)\ $(m(x)+m(y)=m(x+y))\supset (x\cdot y=e)$.

(m11)\ $Cxy\supset (x\hookrightarrow y = 1)$.

(m12)\ $Cxy\supset (x\cdot (-y)=e)$.

(m13)\ $Cyx\supset (m(x\hookrightarrow y)=1-m(x)+m(y))$.

(m14)\ $m(x\hookrightarrow y)=1-m(x-y).$
\end{theorem}
Let us observe that (m13) replicates the \L ukasiewicz implication cf. \L ukasiewicz and Tarski \cite{LukTar}. We now apply the notion of weight toward definitions of
uncertain terms.

\section{Rough Mereology: A theory of uncertain terms}\label{rm} In this section, we develop this topic going in the footsteps of \L ukasiewicz \cite{Lukasiewicz 3} who
attached weights to unary formulae over a finite universe. Our version yields an abstract form.\

The weight function induces a measure of part to a degree called also rough inclusion cf. \cite{Polkowski 2} denoted $P_rxy$ (read: {\em x is a part to degree r of y}).
\begin{definition}\label{rinc} [Part to a degree] $P_rxy\supset (\frac{m(x\cdot y)}{m(x)})=r) \wedge ((\frac{m(x\cdot y)}{m(x)}=r)\supset P_rxy)$.\end{definition}
\begin{theorem}\label{rinc prop} Definition \ref{rinc} yields the following properties of rough inclusions cf.\cite{Polkowski 2}:\

1.\ $P_1ab$ if and only if $Cab$.

2.\ $P_1ab\wedge P_rca\subset \exists s\geq r. P_scb$.\end{theorem}

We impose on approximate mereology the structure of many-valued logic. We begin with the [0,1]-lattice of \L ukasiewicz and its operators. A bit of introduction is in order.\
The crucial object to begin with is a t-norm cf. Menger \cite{Menger}, also Hajek \cite{Hajek} or Polkowski \cite{Polklogic}.
\begin{definition}\label{tnorm} [t-norm] A t-norm $T:[0,1]\rightarrow [0,1]$ is a function conforming to conditions:\

1.\ $T(x,y)=T(y,x)$.

2.\ $T(x, T(y,z))=T(T(x,y),z)$.

3.\ $x_1\leq x_2\supset T(x_1,y)\leq T(x_2,y)$.

4.\ $T(x,0)=0, T(x,1)=x$. \end{definition}
We apply $T_L$,  the \L ukasiewicz t-norm.
\begin{definition}\label{luk} [the \L ukasiewicz t-norm]
 The \L ukasiewicz t-norm is $T_L(x,y)=max\{0, x+y-1\}$; the derived operators are: strong conjunction $x\& y= T_L(x,y)$, conjunction $x\wedge y = min\{x,y\}$, disjunction
 $x\vee y= max\{x,y\}$, strong disjunction $x\underline{\vee} y= S_L(x,y)= min\{1, x+y\}$ (it is $T_L$-co-norm), $x\supset_L y = min\{1, x-y+1\}$, negation  $-_Lx=1-x$.
\end{definition}
We apply these operators to terms of rough mereology; the semantics we propose is as follows. As semantic operators, we adopt the universal valid for each t-norm formulae;
t-norm $T$ induces the strong conjunction $\&=T$ and its residuum $\supset_T$.\

Other operators are defined via formulae:
\begin{enumerate}
\item\ $\neg p\equiv p\supset 0$, where 0 denotes the value zero and at the same time serves as the constant.
\item\ $ p\wedge q\equiv p\& (p\supset q)$.
\item\ $p\vee q\equiv [(p\supset q)\supset q]\wedge [(q\supset p)\supset p]$.
\item\ $p\underline{\vee}q\equiv \neg(\neg p \&\neg q)$.
\end{enumerate}
These operators bear on our terms as follows, where $T$ stands for $T_L$ (clearly, these formulae are universal for any t-norm $T$):\

1.\ $P_rca\wedge P_scb\supset P_{max\{r,s\}}c(a+b)$ (the sum).

2.\ $P_rca\wedge P_scb\supset P_{r\underline{\vee}s}c(a\oplus b)$ (the strong sum).

3.\ $P_rca\wedge P_scb\supset P_{min\{r,s\}}c(a\cdot b)$ (the product).

4.\ $P_rca\wedge P_scb\supset P_{T(r,s)}c(a\odot b)$ (the strong product).

5.\ $P_rca\wedge P_scb\supset P_{T(r,s)}c(a\supset_X b)$ (the implication).

6.\ $\neg P_rca \supset P_{-_Tr}c-a$ (the negation).

\section{Applications to uncertain knowledge: Approximate decisions forecasting in Open World setting} We consider a decision system $\mathcal{D}$ = $(O, F,d,V, V_d)$ cf.
sect.\ref{RST} under the Open World Assumption which means that the set $O$ admits new objects whose sets of feature values have been as of yet unseen.\

The problem is to {\em ascertain the decision prediction ability} of the system  $\mathcal{D}$.\ Value and decision prediction are important problems  in  vital areas of
engineering, business, and, artificial intelligence, and machine learning, see, e.g.,  Mohri \cite{Mohri}, Cesa-Bianchi-Lugosi \cite{CB}, Ray et al.\cite{Ray} as examples
from among many others. We propose an approach applicable to decision systems from the rough set theory point of view which uses CV dimension and similarity measure induced
from rough mereology.\ The scenario we envision is the following.\

Predicting agents: Objects in the decision system $\mathcal{D}$.\

To be predicted: decision values for new yet unseen objects with their values of features in the feature set $F$.\

The expert giving the exact value of decision.\

We are close to the setting of online learning cf. \cite{Mohri}.\

{\bf Our protocol}\

1. We consider a new object $\omega$ with $F$-value set $A=\{\omega_f=f(\omega): f\in F\}$.\

2. For each object $o\in O$, we define the touching set $T_A(o,\omega)=\{\omega_f: \omega_f=f(o)\}$.\

3. We propose to apply to our problem the complexity measure of families of concepts (in our case sets) known as the Vapnik - Chervonenkis dimension, denoted VCdim,
\cite{Vapnik}, see also, e.g., Mohri et al. \cite{Mohri}. To the best of our knowledge it is the first application of  complexity measure to the task of decision's value
assignment in rough set theory.\ We now include the appropriate definitions.
\begin{definition}\label{shat} [Shattering] We say that a family of sets $\mathcal{F}$ shatters a non-empty set $S$ if and only if for each subset $T$ of $S$ there exists
$f\in {\mathcal{F}}$ with the property that $f\cap S=T$.\end{definition}
\begin{definition}\label{VC} [Vapnik-Chervonenkis dimension] In the notation of Def.\ref{shat}, if there exists a natural number $m$ with the property that $\mathcal{F}$
shatters a set of cardinality $m$ and it doesn't shatter any set of cardinality $m+1$ then VCdim($\mathcal{F})=m$; else VCdim($\mathcal{F})= \infty$.\end{definition}
4.\ We establish a parameter $\varepsilon\in [0,1]$ and we define the notion of an $\varepsilon$-component.
\begin{definition}\label{ec}[$\varepsilon$-component] A set C is an $\varepsilon$-component of a set $X$, when the relation $P_{\varepsilon}CX$ holds with respect to the
weight assignment $m(Z)=\frac{|Z|}{|A|}$ for $Z\subseteq A$ and inclusion degree $\mu(X,Y)=\frac{|X\cap Y|}{|X|}$.\end{definition}
5. For each $o\in O$, we denote by the symbol ${\mathcal{C}}_{\varepsilon}(o)$ the set of $\varepsilon$-components in $A$  of the set $T_A(o,\omega)$.\

6. By the symbol VC(o) for $o\in O$, we denote the Vapnik-Chervonenkis dimension of the collection of sets ${\mathcal{C}}_{\varepsilon}(o)$. The symbol VC$^*$ will denote the
maximal value of VC(o). \

7. We now consider the case of consistent decision system $\mathcal{D}$. In this case, each indiscernibility class $[o]_F=\{o': \forall f\in F. f(o)=f(o')\}$ is contained in
the decision class $[o]_d =\{o': d(o)=d(o')\}$.\

8. We select a parameter $\delta$, being a positive natural number, and for each $o\in O$, we define  the prediction of decision for $\omega$ as  $\hat{d}(o) \in N(d(o),
\lfloor\delta\cdot \frac{VC(o)}{VC^*}\rfloor$), $N(o,r)$ meaning the neighborhood of $o$ of radius $r$ in $V_d$. It follows that when an object misses $\omega$, i.e.,
$T_A(o,\omega)=\emptyset$, its prediction is its decision value.\

9. We now input the expert decision $d^{exp}$, and we say that $d^{exp}$ is agreeable with $\hat{d}(o)$ if $d^{exp}\in N(d(o), \lfloor\delta\cdot \frac{VC(o)}{VC^*}\rfloor)$
in which case we assign the reward $R(o)$ of $1$ to $o$, otherwise, the reward $R(o)$ is $0$. \

10. We say that the system $\mathcal{D}$ $\varepsilon, \delta$ - {\em approximately} predicted the decision value on a set $\Omega$ =$\{\omega_i: i=1, 2, \ldots, \omega_n\}$
if $min_i\sum_o R_i(o)\geq 1$, where $R_i(o)$ is the reward of $o\in O$ in the trial of $\omega_i$. This criterion means that in each trial for any $\omega_i$  there is an
object $o_i$ with $R(o_i)=1$.

\begin{remark} 1. Our approach is based on the idea of {\em reasoning by analogy}, hence, we adhere to the view that the larger the touching set, the better prediction
ability of the object. This assumption may not necessarily turn to be  sound in which case we may simply regard objects as experts predicting the decision as their own
decision plus a real expert giving the correct decision (i.e., the reliable case) and apply the Halving Algorithm cf. \cite{Mohri} for it.\

2. We follow he general ideas of online learning, in particular one may trace in our approach some echoes of Weighted Majority Algorithm by Littlestone and Warmuth
\cite{Litt}, cf. \cite{Mohri}.\

3. In case of a non-consistent decision system $\mathcal{D}$, we may define indiscernibility with respect to the set $F\cup \{d\}$ of features and made it into a consistent
one.\

4. Parameters $\varepsilon, \delta$ may be tried in various configurations, e.g., instead of $\varepsilon$ - components, one may consider all $\eta$ - components for
$\eta\geq \varepsilon$ which may be important in the discrete case and changing $\delta$ towards $0$ will lead towards the Halving Algorithm. \

5. The usage of neighborhoods stresses the approximate character of the algorithm and clearly improves the predicting ability of the system, however, the degree of
approximation depends on our choice of $\delta$ which may be arbitrarily small. Let us observe that VC-dimension of the family ${\mathcal{N}}=\{(N(d(o), \lfloor\delta\cdot
\frac{VC(o)}{VC^*}\rfloor)\}$ is at most $2\cdot \delta$ and the loss is less than $2\cdot \delta$.\

6. Concerning the parameter $\varepsilon$, let us observe that:\

6.1\  $\varepsilon =0$. Then, the shattered set is  $|A\setminus T_A(o)|$, hence, $ VC* = max|A\setminus T_A(o)|$ = $|A|-min_o |T_A(o)|$ which prefers small touching sets and
the acceptable loss for larger touching sets  becomes of order of $\delta\cdot \frac{|A\setminus T_A(o)|}{VC^*}$.\

6.2\ $\varepsilon =1$.  Then,  $VC(o)=|T_A(o)|$, thus, $VC^*=max_o |T_A(o)|$, hence, the decisive role in assigning neighborhood sizes is played by cardinalities of touching
sets. The loss error is at most $\delta$ and it is set to values respectively smaller for less numerous touching sets: the worse agent, the more restrictive condition for its
prediction.\

6.3\ For $\varepsilon \in (0, 1)$, we need to weigh the influence of cardinalities of sets $A, T_A(o)$. It follows that to satisfy the condition for $\varepsilon$-component,
a set $Q$ as a candidate should satisfy the inequality $$|Q|\leq max\{\lfloor\frac{|T_A(o)|}{\varepsilon}\rfloor, \lceil\frac{|A\setminus T_A(o)|}{1-\varepsilon}\rceil\}$$
(we have applied floor and ceiling as the quotients in the formula need not be integers). Accordingly $VCdim (T_A(o)) \leq |Q|$.\

As elementary analysis shows, an optimal $\varepsilon$ is $\Theta(sqrt[\frac{|T_A(o)|}{|T_A(o)|\setminus |A\setminus T_A(o)|}])$ if the condition $$|T_A(o)| > |A\setminus
T_A(o)|$$ is satisfied. It follows that in this case our approach prefers larger touching sets which is consistent with intuition.\

7. In case when there are for the given $\omega$ more than one $o\in O$ with the reward $R(o)=1$, we select as the decision pointed to by the assembly of $O$ as the
$\hat{d}(\omega)$ = $\hat{d}(o)^*$ with the minimal loss  $|d^{exp}-\hat{d}(o)|$; in case this minimum is obtained for more than one $o$, the decision is taken according to a
chosen strategy, e.g., the random tie resolution.\ Due to convexity of the value set $V_d$, we may apply the idea of weighted prediction and define the predicted value as
$$\hat{d}(\omega)=\sum_o\frac{\hat{d}(o)\cdot \frac{VC(o)}{VC^*}}{\sum_o\frac{VC(o)}{VC^*}}.$$

8.\ The factor $\frac{VC(o)}{VC^*}$ in the radius definition may be regarded as the counterpart of the weights in the Weighted Average Prediction protocol, cf. \cite{CB}.\

9.\ The notion of regret, see Hannan \cite{Ha}, cf. \cite{Mohri}, \cite{CB}, for the single trial with $\omega$, may be defined as
$$|d^{exp} - \hat{d}(\omega)| -  min_o \{|d^{exp} - \hat{d}(o)|\}.$$
\end{remark}

\section{The mistakes problem} By a mistake we understand the event when an object $o$ fails to satisfy the condition $d^{exp}\in N(d(o), \lfloor\delta\cdot
\frac{VC(o)}{VC^*}\rfloor)$. We denote by the symbol $\mathcal{H}$ the collection $\{N(d(o), \lfloor\delta\cdot \frac{VC(o)}{VC^*}\rfloor): o\in O\}$. As
$\bigcup{\mathcal{H}}$ covers the decision set $V_d$, there must be at least one object $o_m\in O$ which is mistake-free. Hence the number $Mist$ of mistakes is at most
$|O|-1$.\

We assume that we are in the supervised learning mode and the supervisor informs each object $o$ whether its decision has been correct or not. We consider a sequence of
weighted votes. In each vote, the incorrect agents are dismissed and the remaining objects have their neighborhoods radii diminished by a learning rate $\eta$ and these
objects select new decision values $\hat{d}^{1}(o)$. The procedure repeats until the empty set is obtained in which case the fore last set is the set localizing the true
value of decision.

{\bf Conclusion} We have presented classical concept theories Syllogistic and Mereology as preludes to the theory of Rough Mereology. Using notions of Rough Mereology, we
have defined $\varepsilon$-components in order to allow us to define representatives for new objects coming to the decision system from the outside world and in consequence
to approximate the assignment of a decision value to each such object.\ Further research may be empirical, for instance regarding the new object as coming from the
leave-one-out procedure for $O$.

\end{document}